\documentclass[journal]{IEEEtran}
\usepackage{color}
\usepackage{cite}

\ifCLASSINFOpdf
\else
\fi
\usepackage{amssymb,amsmath,amsthm}       %
\usepackage{amsfonts}      %

\newtheorem{lemma}{Lemma}

\newtheorem{example}{Example}

\usepackage{mathtools}     %

\begin{document}
%
\title{Diversity Enhancement via Magnitude}
%
%
%

\author{Steve~Huntsman
\thanks{S. Huntsman is with Systems and Technology Research, Arlington,
VA, 22203 USA e-mail: steve.huntsman@str.us.}
\thanks{Manuscript received January 1, 1970; revised January 19, 2038.}}

%
%

\markboth{Journal of \LaTeX\ Class Files,~Vol.~14, No.~8, January~2022}%
{Huntsman: Diversity Enhancement via Magnitude}
%



\maketitle

\begin{abstract}
Promoting and maintaining diversity of candidate solutions is a key requirement of evolutionary algorithms in general and multi-objective evolutionary algorithms in particular. In this paper, we use the recently developed theory of \emph{magnitude} to construct a gradient flow and similar notions that systematically manipulate finite subsets of Euclidean space to enhance their diversity, and apply the ideas in service of multi-objective evolutionary algorithms. We demonstrate diversity enhancement on benchmark problems using leading algorithms, and discuss extensions of the framework.
\end{abstract}


%
\IEEEpeerreviewmaketitle

\section{\label{sec:Introduction}Introduction}

Promoting and maintaining diversity of candidate solutions is a key requirement of \emph{evolutionary algorithms} (EAs) in general and \emph{multi-objective EAs} (MOEAs) in particular \cite{eiben2015introduction,basto2017survey}. Many ways of measuring diversity have been considered, and many shortcomings identified \cite{yan2007diversity}. Perhaps the most theoretically attractive diversity measure, used by \cite{ulrich2010defining,ulrich2011maximizing}, is the \emph{Solow-Polasky diversity} \cite{solow1994measuring}. It turns out that a recently systematized theory of diversity in generalized metric spaces \cite{leinster2021entropy} singles out the Solow-Polasky diversity or \emph{magnitude} of a (certain frequently total subset of a) finite metric space as equal to the maximum value of the ``correct'' definition \eqref{eq:diversity} of diversity that uniquely satisfies various natural desiderata.
\footnote{
In fact, \cite{solow1994measuring} is the initial appearance of the magnitude concept.
}
While the notion of magnitude was introduced in the mathematical ecology literature over 25 years ago, an underlying notion of a diversity-maximizing probability distribution is much more recent and to our knowledge has not yet been applied to EAs.

In the context of MOEAs, a practical shortcoming associated with magnitude is the typical $O(n^3)$ cost of matrix inversion. To avoid this, \cite{ulrich2010defining,ulrich2011maximizing} resort to a heuristic approximation for the sake of efficiency and merely \emph{measure} diversity rather than attempting to \emph{enhance} it from first principles.

However, it is profitable to incur the marginal cost of computing a so-called weighting \emph{en route} to the magnitude, since we can use a weighting to enhance diversity near the boundary of the image of the candidate solution set under the objective functions. The nondominated part of this image is the current approximation to the Pareto front, and the ability of weightings to couple both diversity and convergence to the Pareto front dovetails with recent indicator-based EA approaches to Pareto-dominance based MOEAs \cite{zitzler2004indicator,wang2019diversity}. 
\footnote{
NB. Common performance indicators induce high density on the boundary of and singularities of the Pareto front \cite{wang2019diversity}. We expect weightings to avoid at least the first of these concentration phenomena. 
}
Moreover, the agnosticism of weightings to dimension further enhances their suitability for such applications.

In this paper, we construct a gradient flow and similar mechanisms that systematically manipulate finite subsets of Euclidean space to enhance their diversity, which provides a useful primitive for quality diversity \cite{pugh2016quality}. We then apply these ideas in service of multi-objective evolutionary algorithms by diversifying solution data through local mutations. For the sake of illustration, we only perform these mutations on the results of a MOEA, though they generally can and should be performed during the course of evolution.

The paper is organized as follows. In \S \ref{sec:Preliminaries}, we introduce the basic concepts of weightings, magnitude, and diversity. In \S \ref{sec:Scales}, we identify an efficiently computable ``positive cutoff'' scale above which a weighting is guaranteed to be proportional to the unique diversity-maximizing distribution. In \S \ref{sec:GradientFlow}, we deveop a notion of a weighting gradient (estimate) and an associated gradient flow. In \S \ref{sec:EnhancingDiversity}, we use this gradient flow to demonstrate diversity enhancement on a toy problem before turning to benchmark problems in \S \ref{sec:Benchmark}. Finally, we discuss algorithmic extensions and discrete analogues in \S \ref{sec:Extensions} and \S \ref{sec:DiscreteObjective}, respectively, before making brief remarks in \S \ref{sec:Remarks}. Appendix \S \ref{sec:erosion} discusses an ``erosion'' procedure that is relevant to culling a population.

\section{\label{sec:Preliminaries}Weightings, magnitude, and diversity}

For details on the ideas in this section, see \S 6 of \cite{leinster2021entropy}.

Call a square nonnegative matrix $Z$ a \emph{similarity matrix} if its diagonal is strictly positive. An important class of similarity matrices is of the form $Z = \exp[-td]$ where the exponential is componentwise,
\footnote{
For a matrix $M$ and well behaved function $f$, we follow a standard practice in writing $(f[M])_{jk} := f(M_{jk})$ to distinguish a function applied to matrix entries versus to the matrix itself in the sense of functional calculus. Thus, e.g., $\exp[M] \ne \exp(M) = I + M + M^2/2! + \dots$. 
}
$t \in (0,\infty)$, and $d$ is a square matrix whose entries are in $[0,\infty]$ and satisfy the triangle inequality. In this paper, $d$ will always be the matrix of distances for a finite subset of Euclidean space unless otherwise specified.

A \emph{weighting} $w$ is a column vector satisfying $Zw = 1$, where the vector of all ones is indicated on the right. 
\footnote{
In Euclidean space, $Z = \exp[-td]$ is a radial basis function interpolation matrix \cite{buhmann2003radial} and the weighting equation $Zw = 1$ amounts to the statement that the weighting $w$ provides the coefficients for interpolating the unit function. That is, if $\{x_j\}_{j=1}^n$ are points in Euclidean space with distance matrix $d$ and we have a weighting $w$ satisfying $\sum_k w_k \exp(-td_{jk}) = 1$, then in fact $u(x) := \sum_k w_k \exp(-t|x-x_k|) \approx 1$, where $\approx$ indicates an optimal interpolation in the sense of a representer theorem \cite{scholkopf2001generalized}. In particular, the triangle inequality yields $u(x_j + \delta x_j) \ge \exp(-t|\delta x_j|)$. 
}
A \emph{coweighting} is the transpose of a weighting for $Z^T$. If $Z$ admits both a weighting $w$ and a coweighting, then its \emph{magnitude} is defined via $\text{Mag}(Z) := \sum_j w_j$, which also turns out to equal the sum of the coweighting components.

In the case $Z = \exp[-td]$ and $d$ is the distance matrix corresponding to a finite subset of Euclidean space, $Z$ is positive definite, hence invertible, and so its weighting and magnitude are well-defined and unique. 
\footnote{
\label{foot:positiveDefinite}
In Euclidean space, $Z = \exp[-td]$ is just a Laplacian kernel, which is positive definite \cite{steinwart2008support}. 
}
More generally, if $Z$ is invertible then $\text{Mag}(Z) = \sum_{jk} (Z^{-1})_{jk}$. For a generalized metric space $d$ of the form specified above, the \emph{magnitude function} $\text{Mag}(t;d)$ is defined as the map $t \mapsto \text{Mag}(\exp[-td])$.  

It turns out \cite{willerton2009heuristic,bunch2020practical} that weightings are excellent scale-dependent boundary detectors in Euclidean space (see, e.g., Figure \ref{fig:boundaryABC20210402}). In particular, if $A$ is compact and $B \subset A$ is reasonably nice and finite (e.g., a uniform random sample or intersection with a regular lattice), then the largest components of a weighting on $B$ tend to occur near the boundary of $A$. Meanwhile, any negative weighting components (if they exist) tend to be ``just behind the boundary'' of $A$ (see \S \ref{sec:erosion}). A technical explanation for this boundary-detecting behavior draws on the potential-theoretical notion of Bessel capacities \cite{meckes2015magnitude}. 
\footnote{
\label{foot:Bessel}
For the case of an interval in $\mathbb{R}$ this has been examined in detail: see the discussion following Proposition 5.9 of \cite{meckes2015magnitude}. More generally, the weighting of compact $A \subset \mathbb{R}^n$ (defined for infinite $A$ via a suitable technical procedure) is the distribution $\frac{1}{n! \omega_n}(I-\Delta)^{(n+1)/2}h$, where $\omega_n$ is the volume of the unit ball in $\mathbb{R}^n$ and $h$ is the \emph{Bessel potential function} of $A$. The Bessel potential function is defined in turn as the function that takes the value $1$ on $A$ and minimizes the \emph{Bessel potential space} norm, which generalizes the notion of a Sobolev norm. Another less precise but more computationally expedient manifestation of the same idea arises from an identity of Varadhan \cite{varadhan1967behavior} that is equivalent to $(I-[\frac{n+1}{t}]^2 \Delta)^{-(n+1)/2} = \exp (-td + o(t^{-1}))$ for manifolds. Using a fast algorithm for the Laplacian can be computationally favorable, as discussed in Remark 4.19 of \cite{peyre2019computational}.
}
Our own investigations of asymmetric and thus manifestly non-Euclidean distances suggest that (co)weightings also indicate the presence of boundary-like features in those settings. 

Meanwhile, magnitude is a very attractive and general notion of size that encompasses both cardinality and Euler characteristic, as well as encoding other rich scale-dependent geometrical data \cite{leinster2017magnitude}. 

\begin{example}
\label{ex:3PointSpace}
Consider $\{x_j\}_{j=1}^3 \subset \mathbb{R}^2$ with pairwise distances $d_{jk} := d(x_j,x_k)$ given by $d_{12} = d_{13} = 1 = d_{21} = d_{31}$ and $d_{23} = \delta = d_{32}$ with $\delta<2$. A routine calculation yields that 
\begin{equation}
w_1 = \frac{e^{(\delta+2)t}-2e^{(\delta+1)t}+e^{2t}}{e^{(\delta+2)t}-2e^{\delta t}+e^{2t}} \nonumber
\end{equation}
and
\begin{equation}
w_2 = w_3 = \frac{e^{(\delta+2)t}-e^{(\delta+1)t}}{e^{(\delta+2)t}-2e^{\delta t}+e^{2t}}. \nonumber
\end{equation}
This is shown in Figure \ref{fig:3pointSpace20210402} for $\delta = 10^{-3}$. At $t = 10^{-2}$, the ``effective size'' of the nearby points is $\approx 0.25$, and that of the distal point is $\approx 0.5$, so the ``effective number of points'' at this scale is $\approx 1$. At $t = 10$, these effective sizes are respectively $\approx 0.5$ and $\approx 1$, so the effective number of points is $\approx 2$. Finally, at at $t = 10^4$, the effective sizes are all $\approx 1$, so the effective number of points is $\approx 3$. 


\begin{figure}[h]
  \centering
  \includegraphics[trim = 40mm 110mm 40mm 105mm, clip, width=\columnwidth,keepaspectratio]{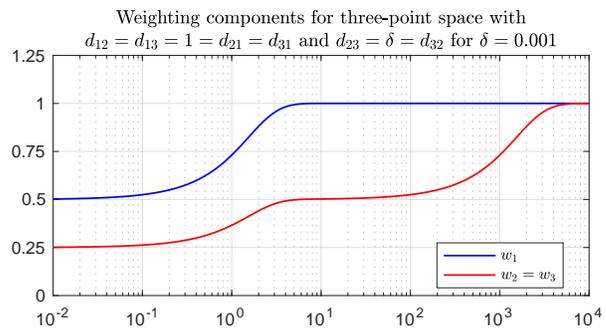}
  \caption{Weighting components for an ``isoceles'' metric space. The magnitude function $w_1+w_2+w_3$ gives a scale-dependent ``effective number of points.''}
  \label{fig:3pointSpace20210402}
\end{figure}
\end{example}

The notion of magnitude has been used by ecologists to \emph{quantify} diversity since the work of Solow and Polasky \cite{solow1994measuring}, but much more recent mathematical developments have clarified the role that magnitude and weightings play in \emph{maximizing} a more general and axiomatically supported notion of diversity \cite{leinster2016maximizing,leinster2021entropy}. Specifically, the \emph{diversity of order $q$} for a probability distribution $p$ and similarity matrix $Z$ is the exponent of 
\begin{equation}
\label{eq:diversity}
\frac{1}{1-q} \log \sum_{j: p_j > 0} p_j (Zp)_j^{q-1}
\end{equation} 
for $1 < q < \infty$, and via limits for $q = 1,\infty$.
\footnote{
The expression \eqref{eq:diversity} is a ``similarity-sensitive'' generalization of the R\'enyi entropy of order $q$. In the event $Z = I$, the usual R\'enyi entropy is recovered, with Shannon entropy as the case $q = 1$.
}
This is the ``correct'' measure of diversity in essentially the same way that Shannon entropy is the ``correct'' measure of information. With this in mind, we restrict our attention to it versus other measures such as those discussed in \cite{yan2007diversity,basto2017survey}.

If $Z$ is symmetric, then the (unique) positive weighting of the submatrix on common row and column indices that has the largest magnitude is proportional to the diversity-saturating distribution for \emph{all} values of the free parameter $q$ in the definition, and this magnitude equals the maximum diversity. In general the diversity-maximizing distribution is $\mathbf{NP}$-hard to compute, though cases of size $\le 25$ are easily handled on a laptop. Fortunately, the situation improves radically if besides being symmetric, $Z$ is also positive definite (as guaranteed for the metric on finite subsets of Euclidean space) and admits a positive weighting. Then this (unique) weighting is proportional to the diversity-maximizing distribution, and standard linear algebra suffices to obtain it efficiently. 

Using an efficiently computable ``diagonal cutoff'' scale factor $t = t_d$ defined in Lemma \ref{lemma:diagonalCutoffBounds}, we can optimally enforce this advantageous situation for similarity matrices of the form $Z = \exp[-td]$. With $t_d$ in hand, we can also efficiently identify a ``positive cutoff,'' i.e., the smallest scale $t_+ \le t_d$ at which the weighting is automatically guaranteed to be proportional to a diversity-maximizing distribution. 

This in turn allows us to define a gradient flow that increases the diversity of finite subsets of Euclidean space. This gradient flow is a  principled and computationally efficient primitive for diversity optimization that informs evolutionary algorithms in Euclidean spaces. More generally, this gradient flow can also be defined on ``positive definite metric spaces'' \cite{meckes2013positive} with vector space structure. Stochastic mechanisms that increase diversity on essentially arbitrary spaces satisfying a triangle inequality can be constructed as in \S \ref{sec:DiscreteObjective}.

\section{\label{sec:Scales}Cutoff and zero scales}

In this section, we introduce two relevant scale factors associated with finite metric spaces, most notably a ``positive cutoff'' $t_+$ that is the minimal scale factor above which a weighting is proportional to a diversity-maximizing distribution. Besides allowing us to efficiently link weightings and diversity, using this scale factor also allows us to eliminate an unwanted degree of freedom in the definition of a weighting. For completeness, we also outline the behavior of scale factors approaching infinity and zero.

For similarity matrices that are ultrametric
\footnote{
The ultrametric triangle inequality is $d(x,z) \le \max \{d(x,y),d(y,z)\}$. Given a nonnegative symmetric matrix with zero diagonal, we can construct the unique maximal subdominant ultrametric (viz., the maximal hop in a minimal-weight path) using a spanning tree construction as in \S IV.C of \cite{rammal1986ultrametricity}. 
}
or diagonally dominant
\footnote{
Recall that a square matrix over $\mathbb{R}$ is diagonally dominant if each diagonal element exceeds the sum of the other entries in its row.
}
and of the form $\exp[-d]$ with $d$ symmetric, nonnegative, and with zero diagonal, there is a polynomial-time algorithm to compute the diversity-maximizing distribution that is also practical and admits ample scope for acceleration. This algorithm is just to normalize the weighting.

For $d$ with zero diagonal and all other entries positive, there exists a minimal $t_d >0$ such that $\exp[-td]$ is diagonally dominant for any $t > t_d$. We call $t_d$ the \emph{diagonal cutoff}. Because $\exp[-td]$ is diagonally dominant iff $1 > \max_j \sum_{k \ne j} \exp(-t d_{jk})$, we can efficiently estimate $t_d$ to any desired precision using the following elementary bounds and a binary search:
\begin{lemma} 
\label{lemma:diagonalCutoffBounds} 
For $d \in M_n$ as above,
\footnote{
A metric $d$ on $\{1,\dots,n\}$ is called \emph{scattered} if $d(x,y) > \log(n-1)$ for all $x \ne y$ (see, e.g., Definition 2.12 of \cite{leinster2013magnitude}). A corollary of Lemma \ref{lemma:diagonalCutoffBounds} that amounts to Proposition 2.4.17 of \cite{leinster2013magnitude} is that the diagonal cutoff of a scattered metric space is strictly less than 1.
}
\begin{equation}
\label{eq:diagonalCutoffBounds}
\frac{\log(n-1)}{\min_j \max_k d_{jk}} \le t_d \le \frac{\log(n-1)}{\min_j \min_{k \ne j} d_{jk}}.
\end{equation}
\end{lemma}

\begin{proof}
If $\exp[-td]$ is diagonally dominant, then 
\begin{align}
& 1 > (n-1) \cdot \max_j \min_{k \ne j} \exp(-t d_{jk}) \nonumber \\
\iff & 1 > (n-1) \cdot \exp \left ( -t \cdot \min_j \max_{k \ne j} d_{jk} \right ) \nonumber \\
\iff & \log(n-1) < t \cdot \min_j \max_k d_{jk}, \nonumber 
\end{align}
which yields the first inequality in \eqref{eq:diagonalCutoffBounds}. Meanwhile, $\exp[-td]$ is diagonally dominant if
\begin{align}
& 1 > (n-1) \cdot \max_j \max_{k \ne j} \exp(-t d_{jk}) \nonumber \\
\iff & 1 > (n-1) \cdot \exp \left ( -t \cdot \min_j \min_{k \ne j} d_{jk} \right ) \nonumber \\
\iff & \log(n-1) < t \cdot \min_j \min_{k \ne j} d_{jk}, \nonumber 
\end{align}
which yields the second inequality in \eqref{eq:diagonalCutoffBounds}. 
\end{proof}

As foreshadowed, we can also use Lemma \ref{lemma:diagonalCutoffBounds} to find the \emph{positive cutoff}, defined as the smallest $t_+ < t_d$ such that $\exp[-td]$ admits a positive weighting for $t > t_+$. 
\footnote{
\label{foot:positiveCutoff}
Outside the Euclidean setting, it makes sense to define $t_+$ to be the minimal value of $t$ such that $Z = \exp[-td]$ admits a positive weighting \emph{and $Z$ is positive semidefinite} for $t > t_+$. But in the Euclidean setting, we get positive definiteness for free: see footnote \ref{foot:positiveDefinite}. 
}

The limit $t \uparrow \infty$ of (co)weightings and the magnitude function is uninteresting: the (co)weightings are (co)vectors of all ones, and the limiting value of the magnitude function is thus just the number of points in the space under consideration. However, as Figure \ref{fig:boundaryABC20210402} suggests, the limit $t \downarrow 0$ contains more detailed structural information. The functional form of this limit is not hard to obtain:

\begin{lemma} 
\label{lemma:scale0} 
For $d \in GL(n,\mathbb{R})$ the solution to $\exp[-td] w(t) = 1$ (where as usual the exponential is elementwise) has the well-defined limit $w(0) := \lim_{t \downarrow 0} w(t) = \frac{d^{-1}1}{1^T d^{-1} 1}$. Similarly, the solution to $v(t) \exp[-td] = 1$ has the limit $v(0) = \frac{1^T d^{-1}}{1^T d^{-1} 1}$.
\footnote{
It would be interesting and probably also useful to characterize the quality of the approximation $w(t) \approx 1 + e^{-t} \frac{(d^{-1}-I)1}{1^T d^{-1} 1}$. More generally, the same could be said of the approximation $I + e^{-t} \frac{(d^{-1}-I)}{1^T d^{-1} 1} \approx Z^{-1}$. Probably a decent answer should inform (or be informed by) results on Bessel potentials and/or the Varadhan identity (see footnote \ref{foot:Bessel}).
}
\end{lemma}

\begin{proof}[Proof (sketch)]
We have the first-order approximation $(11^T - td) w(t) \approx 1$. Applying the Sherman-Morrison-Woodbury formula \cite{horn2012matrix} and L'H\^opital's rule yields the result.
\end{proof}

\section{\label{sec:GradientFlow}The weighting gradient flow}

Although there are various sophisticated approaches to estimating gradients on point clouds (see, e.g., \cite{luo2009approximating}), a reasonable heuristic estimate for the specific case of the gradient of a weighting $w$ on $\{x_j\}_j$ in Euclidean space is
\begin{equation}
\label{eq:gradient}
(\hat \nabla w)_j := \sum_{k \ne j} \frac{Z_{jk}}{\sum_{k' \ne j} Z_{jk'}} \frac{w_k - w_j}{d_{jk}} e_{jk}
\end{equation}
where $e_{jk} := \frac{x_k - x_j}{d_{jk}}$.

\begin{example}
\label{ex:GradientsABC}
Figure \ref{fig:boundaryABC20210402} illustrates how weightings identify boundaries at various scales; Figure \ref{fig:GradientsABC20210402} shows the corresponding weighting gradient estimates \eqref{eq:gradient}.

\begin{figure}[h]
  \centering
  \includegraphics[trim = 40mm 110mm 40mm 105mm, clip, width=\columnwidth,keepaspectratio]{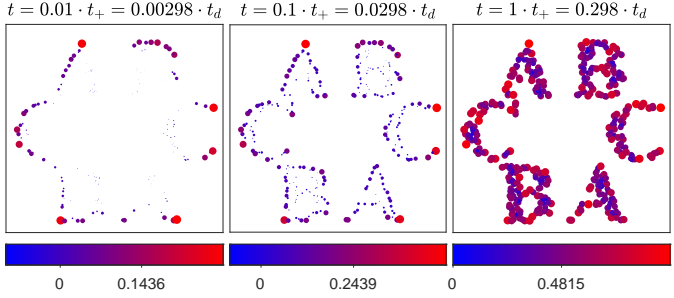}
  \caption{Weighting components for $500$ points sampled without replacement from a probability distribution on $\mathbb{Z}^2$ that is approximately uniform on its support. From left to right, various scale factors $t$ defining a similarity matrix via $Z = \exp[-td]$ are shown in terms of the intrinsic scales $t_d$ and $t_+$ (for which see \S \ref{sec:Scales}). Here $d$ is given by the usual Euclidean distance. Both the color and size of a point are functions of the weighting component; the nonzero color axis tick mark is at half the maximum value.}
  \label{fig:boundaryABC20210402}
\end{figure}



    %
    %
    %
    %

\begin{figure}[h]
  \centering
  \includegraphics[trim = 40mm 115mm 40mm 110mm, clip, width=\columnwidth,keepaspectratio]{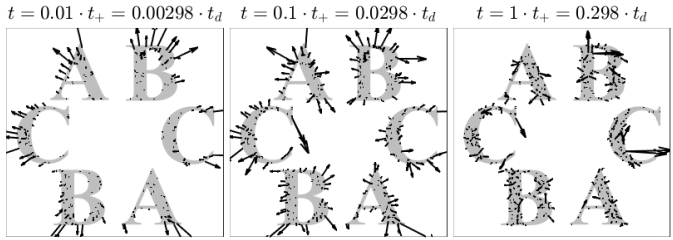}
  \caption{Weighting gradient estimate \eqref{eq:gradient} for the data in Figure \ref{fig:boundaryABC20210402}. The gradient vectors are scaled uniformly in each panel for visualization purposes. Note that for the largest value of $t$ the large gradient vectors have basepoints near other large gradient vectors.}
  \label{fig:GradientsABC20210402}
\end{figure}

\end{example}

It might appear that the \emph{weighting gradient flow} 
\begin{equation}
\label{eq:weightingGradientFlow}
\dot x = \hat \nabla w
\end{equation}
induced by \eqref{eq:gradient} merely drives points apart, but Example \ref{ex:3PointFlow} shows that the phenomenology is more nuanced.
\footnote{
Although the weighting gradient flow \eqref{eq:weightingGradientFlow} is for our purposes a finite system of coupled nonlinear ODEs, it seems likely in light of footnote \ref{foot:Bessel} that this is an analogue of some extrinsic geometric flow in the sense of \cite{andrews2020extrinsic}.
}
\footnote{
Nevertheless, it seems likely that the weighting gradient flow \eqref{eq:weightingGradientFlow} can inform clustering algorithms. We leave this for future work.
}

\begin{example}
\label{ex:3PointFlow} 
Continuing Example \ref{ex:3PointSpace}, assume without loss of generality that $(x_1)_2 = 0 = (x_2)_2+(x_3)_2$. Now
\begin{equation}
(\hat \nabla w)_1 = \frac{\sqrt{4-\delta^2}}{2} \frac{e^{(\delta+1)t}-e^{2t}}{e^{(\delta+2)t}-2e^{\delta t}+e^{2t}} e_1 \nonumber
\end{equation}
and
\begin{equation}
(\hat \nabla w)_2 = \frac{e^{-t}}{e^{-t}+e^{-\delta t}} \frac{e^{2t}-e^{(\delta+1)t}}{e^{(\delta+2)t}-2e^{\delta t}+e^{2t}} e_{21}. \nonumber
\end{equation}

Thus the sign of $\langle (\hat \nabla w)_1, e_1 \rangle$ is the same as that of $\delta-1$, and the sign of $\langle (\hat \nabla w)_2, e_{21} \rangle$ is the opposite. If $\delta = 1$, the flow is stationary. Meanwhile, $\langle (\hat \nabla w)_2, e_1 \rangle = -\frac{\sqrt{4-\delta^2}}{2} \langle (\hat \nabla w)_2, e_{21} \rangle$, so the isoceles triangle both increases its aspect ratio and experiences an overall shift. In other words, the weighting gradient flow wants to flatten \emph{and move} the isoceles triangle in the direction suggested by its initial aspect ratio, as shown in Figures \ref{fig:3PointFlow1_20210222} and \ref{fig:3PointFlow2_20210222}.
\end{example}

\begin{figure}[h]
  \centering
  \includegraphics[trim = 40mm 120mm 40mm 115mm, clip, width=\columnwidth,keepaspectratio]{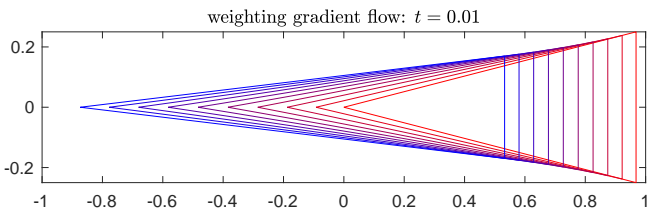}
  \caption{Effect of weighting gradient flow \eqref{eq:weightingGradientFlow} for Example \ref{ex:3PointFlow} with $\delta = 0.5 < 1$: evolution goes from {\color{red}red (right)} to {\color{blue}blue (left)}. (Note that both axis scales are equal.) Increasing $t$ has the qualitative effect of slowing the flow.}
  \label{fig:3PointFlow1_20210222}
\end{figure}

\begin{figure}[h]
  \centering
  \includegraphics[trim = 40mm 115mm 40mm 115mm, clip, width=\columnwidth,keepaspectratio]{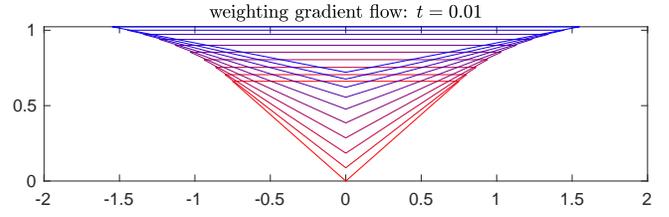}
  \caption{Effect of weighting gradient flow \eqref{eq:weightingGradientFlow} for Example \ref{ex:3PointFlow} with $\delta = 1.5 > 1$ (and switching coordinate axes for convenience): evolution goes from {\color{red}red (bottom)} to {\color{blue}blue (top)}. (Note that both axis scales are equal.) Increasing $t$ has the qualitative effect of slowing the flow.}
  \label{fig:3PointFlow2_20210222}
\end{figure}

\section{\label{sec:EnhancingDiversity}Enhancing diversity}


Following \cite{ishibuchi2012two}, we apply the ideas sketched above to a toy problem where the objective function $f$ has three components, each measuring the distance to a vertex of a regular triangle with vertices in $S^1$. The application is mostly conceptually straightforward, but we mention a few implementation details: 
\begin{itemize}
	\item We begin with a uniformly distributed sample of $n_0 = 10^3$ points in the disk of radius $1.25$, and retain $n$ points that are dominated by no more than $\delta = 0.1$;
	\item Replace any badly behaving points (e.g., out of bounds or \texttt{NaN}s) with their predecessors; 
	\footnote{
		One notion of ``bad behavior'' that we do not curtail is of $f(x+dx)-f(x)$ being a poor approximation of $dy$. The reason is that an effective approximation via Taylor's theorem requires computing the Hessian, which is expensive. Given an alternative route to an effective approximation, we could instill the required discipline in the pullback/recomputation step below.
	}
	\item Introduce a ``speed factor'' $S_j := 1-2\frac{\text{dom}_j}{\max_k \text{dom}_k}$, where $\text{dom}_j$ is the number of points dominating the $j$th point;
	\item Evolve the $n$ points under a modulated weighting gradient flow \emph{on the objective space} with $t = t_+$ as $dy_j = ds \cdot S_j (\hat \nabla w)_j$ for only $N = 10$ steps and step size $ds = \sqrt{\langle \min_{k \ne j} (d_f)_{jk} \rangle/n}$, where the pullback metric on solution space is $d_f(x,x') := d(f(x),f(x'))$;
	\item Pull back the weighting gradient flow from the objective space to the solution space using the pseudoinverse of the Jacobian, then recompute points in the objective space.
\end{itemize}

The result of this experiment is depicted in Figures \ref{fig:ParetoTriSol20210212} and \ref{fig:ParetoTriObj20210212}. The salutary effect on diversity in objective (and solution) space is apparent. This can be quantified via the objective space magnitude functions, as shown in Figures \ref{fig:ParetoTriObjMag20210212} and \ref{fig:yParetoTriObjMag}.

\begin{figure}[h]
  \centering
  \includegraphics[trim = 40mm 105mm 40mm 95mm, clip, width=\columnwidth,keepaspectratio]{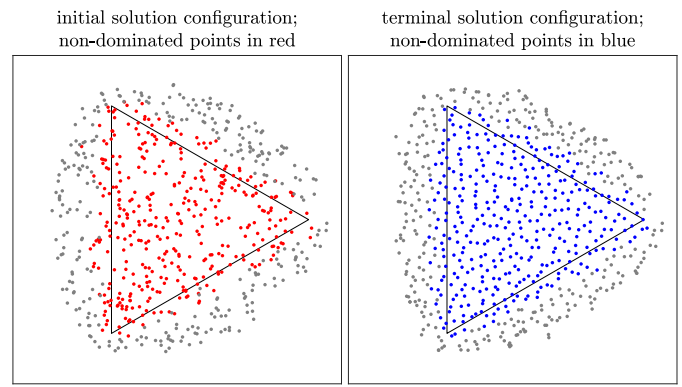}
  \caption{Comparision of {\color{red}initial (red; left)} and {\color{blue}terminal (blue; right)} locations of points in the solution space. The weighting gradient flow produces more evenly distributed terminal points. The triangle defining objective components (by distance to its vertices) is shown. The actual Pareto front is the interior of the triangle and the area displayed is $[-1,1]^2$.}
  \label{fig:ParetoTriSol20210212}
\end{figure}

\begin{figure}[h]
  \centering
  \includegraphics[trim = 30mm 95mm 30mm 90mm, clip, width=\columnwidth,keepaspectratio]{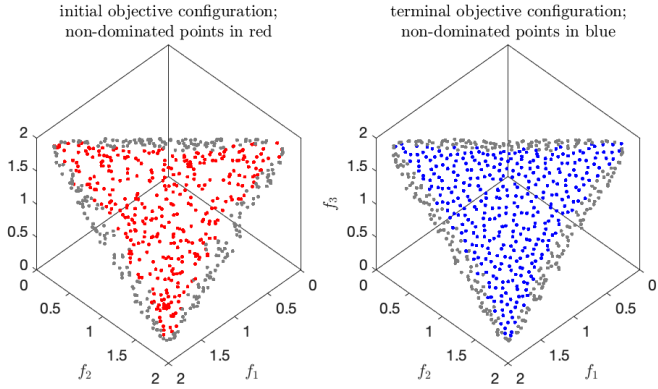}
  \caption{Comparision of {\color{red}initial (red; left)} and {\color{blue}terminal (blue; right)} locations of points in the objective space. The terminal points are more evenly distributed.}
  \label{fig:ParetoTriObj20210212}
\end{figure}

\begin{figure}[h]
  \centering
  \includegraphics[trim = 25mm 100mm 25mm 100mm, clip, width=\columnwidth,keepaspectratio]{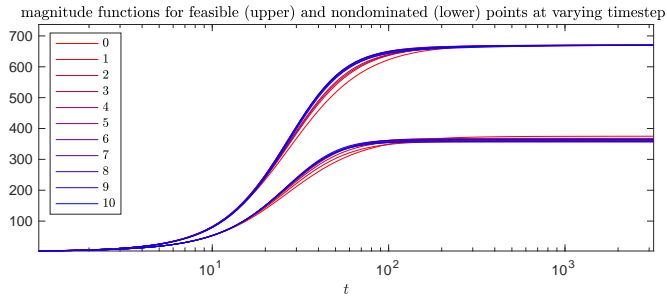}
  \caption{
  Magnitude functions (note the logarithmic horizontal axis) for feasible (upper curves) and nondominated points (lower curves) at timesteps (indicated in legend: the horizontal axis corresponds to scale, not time) from {\color{red}$0$ (red)} to {\color{blue}$N = 10$ (blue)}, under the evolution of the (modulated) weighting gradient flow. 
  Recall from \S \ref{sec:Scales} that magnitude functions always have left and right limits at $1$ and $n$, respectively, so only intermediate values indicate anything meaningful. Here, the magnitude increases significantly over time at scales $t \approx 50$, while for each timestep $t_+ < 32$. The number of nondominated points fluctuates slightly, starting at 375 before decreasing to 357 and increasing back to 364; there are $n = 670$ feasible points. 
  }
  \label{fig:ParetoTriObjMag20210212}
\end{figure}

\begin{figure}[h]
  \centering
  \includegraphics[trim = 10mm 95mm 20mm 100mm, clip, width=\columnwidth,keepaspectratio]{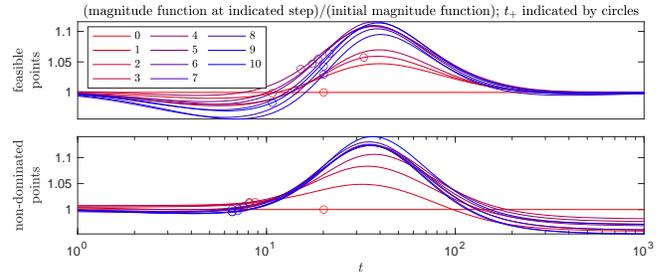}
    \caption{Magnitude increases at scales above $t_+$, where magnitude equals diversity. (Top) Magnitude function quotients ($\text{current}/\text{initial}$; cf. Figure \ref{fig:ParetoTriObjMag20210212}) at various timesteps for feasible points. Timesteps of numerators are indicated via color, going from {\color{red}red at the initial timestep (0)} to {\color{blue}blue at the final timestep ($N = 10$)}: the denominator is the function at the initial timestep. Circles indicate the scales $t_+$. (Bottom) As above, but for non-dominated points.}
  \label{fig:yParetoTriObjMag}
\end{figure}

Using 12 objectives defined by the distance to vertices of a regular dodecagon but otherwise proceeding in exactly the same way yields the solution space result shown in Figure \ref{fig:ParetoDodecaSol20210212}.

\begin{figure}[h]
  \centering
  \includegraphics[trim = 40mm 105mm 40mm 95mm, clip, width=\columnwidth,keepaspectratio]{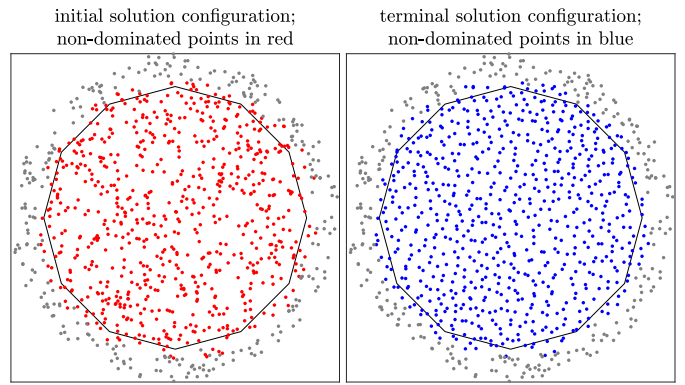}
  \caption{As in Figure \ref{fig:ParetoTriSol20210212} but for a dodecagon.}
  \label{fig:ParetoDodecaSol20210212}
\end{figure}

\section{\label{sec:Benchmark}Performance on benchmarks}


The effectiveness of the (modulated) weighting gradient flow approach hinges on the ability to cover and thereby ``keep pressure on'' the Pareto front. A straightforward way to do this is to use a MOEA to produce an initial overapproximation of the Pareto front as in \cite{guariso2020improving}, and then improve the diversity of the overapproximation via the weighting gradient flow. 
We proceed to detail the results of an experiment along these lines. For the experiment we considered two leading MOEAs (NSGA-II \cite{deb2002fast} and SPEA2 \cite{zitzler2001spea2}) and two leading benchmark problem sets (DTLZ \cite{deb2001scalable} \footnote{
For DTLZ, we considered only the two most relevant problems, viz. DTLZ4 and DTLZ7. DTLZ4 was formulated ``to investigate an MOEA's ability to maintain a good distribution of solutions'' and DTLZ7 was formulated to ``test an algorithm's ability to maintain subpopulation in different Pareto-optimal regions'' \cite{deb2001scalable}. (NB. One approach for this, not pursued here, is to resample points so that the diversity per point in each connected component of the Pareto front is approximately equal. For the application of topological data analysis to Pareto fronts, see \cite{hamada2018data}.) The other problems in the DTLZ set are designed to evaluate other desiderata, e.g. convergence to the Pareto front.
} 
and WFG \cite{huband2006review}), all implemented in PlatEMO version 2.9 \cite{tian2017platemo}. For each problem, we used 10 decision variables, three objectives (to enable visualization), and performed 10 runs (which appears quite adequate for characterization purposes) with population size 250 and $10^4$ fitness evaluations. We then took $N = 10$ timesteps for the weighting gradient flow as before.

Figure \ref{fig:yFirstRunMag_WFG2} (cf. Figure \ref{fig:yParetoTriObjMag}) shows magnitude functions corresponding to various timesteps of the (modulated) weighting gradient flow applied to the results of NSGA-II on the benchmark problem WFG2. Feasible points show a diversity (as measured by magnitude at scale $t_+$ for the feasible objective points) increase of about 10 percent, whereas non-dominated points show a diversity increase of several percent as well, even as the total number of non-dominated points decreases by about 15 percent.

\begin{figure}[h]
  \centering
  \includegraphics[trim = 10mm 95mm 20mm 100mm, clip, width=\columnwidth,keepaspectratio]{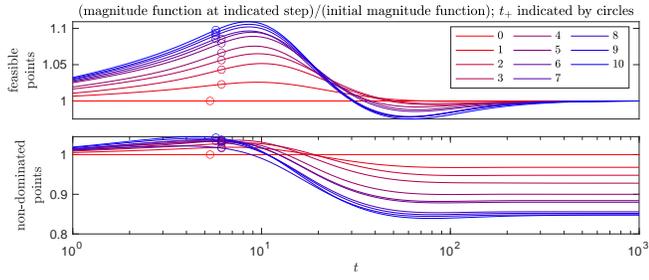}
  \caption{Magnitude increases at scales above $t_+$, where magnitude equals diversity (cf. Figure \ref{fig:yParetoTriObjMag}). (Top) Magnitude function quotients at various timesteps for feasible points under the evolution of the (modulated) weighting gradient flow applied to a solution of the WFG2 benchmark problem obtained via NSGA-II. The horizontal axis $t$ indicates the scale parameter; timesteps of numerators are indicated via color, going from {\color{red}red at the initial timestep (0)} to {\color{blue}blue at the final timestep (10)}: the denominator is the function at the initial timestep. Circles indicate the scales $t_+$. (Bottom) As above, but for non-dominated points. (NB. The corresponding evolution of objective space configurations is shown in Figure \ref{fig:yFirstRunFlow_NSGAII_WFG2}.)}
  \label{fig:yFirstRunMag_WFG2}
\end{figure}

In lieu of showing details such as in Figures \ref{fig:yFirstRunMag_WFG2} for multiple runs and problem instances, we produce an ensemble characterization in Figure \ref{fig:yEnsembleCharacterization}. The figure shows that the number of non-dominated points decreases since the weighting gradient flow pushes some points a short distance away from the Pareto front (as illustrated in Figure \ref{fig:igdEnsembleCharacterization}) before they are halted or reversed. The figure also shows that the diversity of non-dominated points generally increases slightly, and the diversity of feasible points increases significantly. As a consequence, the diversity contributions of individual solutions (as measured by the average weighting, i.e., the magnitude of non-dominated points divided by their cardinality) also increases significantly. For less challenging problems such as shown above for distances to vertices of regular polygons, the number of non-dominated points will decrease less (if at all, cf. Figure \ref{fig:ParetoTriObjMag20210212}), and the diversity gains will be enhanced. 

On the other hand, the positive effects of the weighting gradient flow are considerably reduced in the case of SPEA2, which produces a visibly more uniform 
\footnote{
As a reminder, uniformity and diversity are not the same thing.
}
distribution in objective space than NSGA-II for all of the problems we consider: see Figures \ref{fig:yFirstRunFlow_SPEA2_WFG2}-\ref{fig:yFirstRunFlow_SPEA2_WFG3}. Indeed, the weighting gradient flow appears to \emph{decrease} this uniformity; the formation of a gap just behind the boundary along with a slight increase in the population near the boundary are the main visible indicators that something useful (at least for DTLZ4, WFG2, WFG3, WFG6, and WFG8, per Figure \ref{fig:yEnsembleCharacterization}) is actually happening. 

\begin{figure}[h]
  \centering
  \includegraphics[trim = 10mm 115mm 10mm 115mm, clip, width=\columnwidth,keepaspectratio]{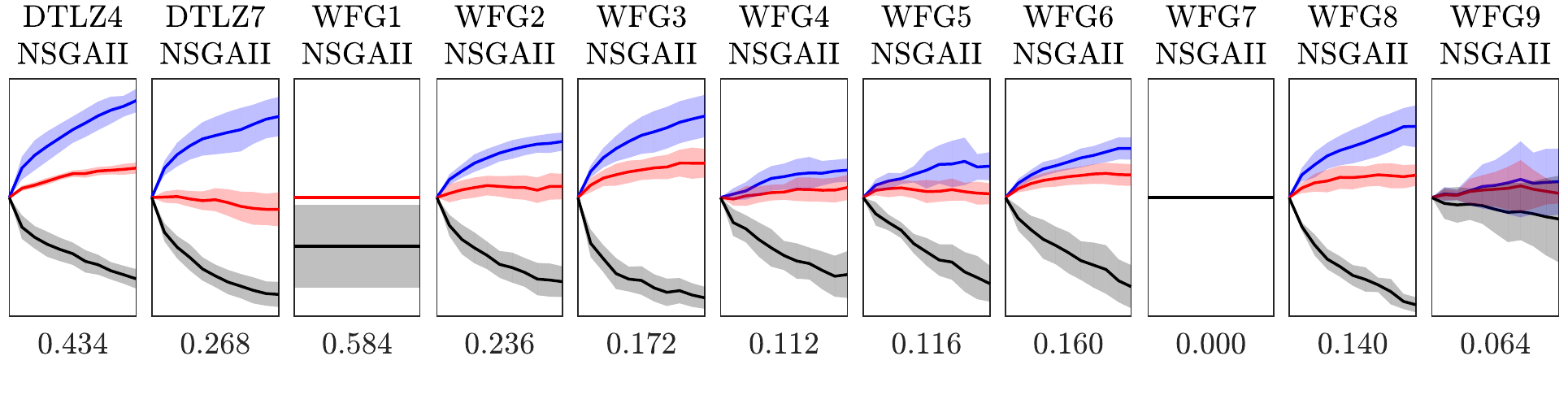} \\
  \includegraphics[trim = 10mm 120mm 10mm 120mm, clip, width=\columnwidth,keepaspectratio]{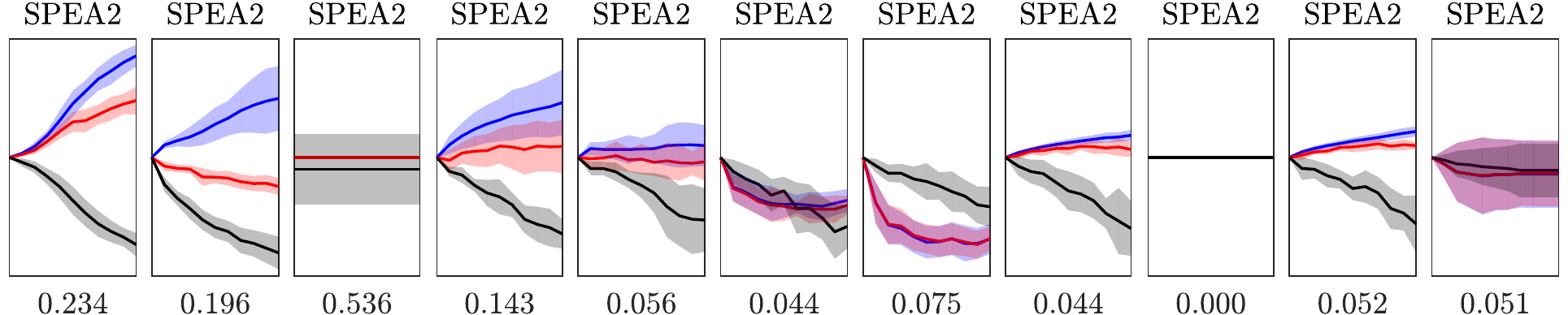} 
  \caption{Diversity of solutions increases markedly under the weighting gradient flow, even as some points become slightly dominated. (Top) Average diversity quotients of {\color{blue}feasible (blue)} and {\color{red}non-dominated (red)} points under the weighting gradient flow along with proportion of population that remains non-dominated (black). Here the diversity is the magnitude at scale $t_+$: for instance, the diversity quotients of feasible and non-dominated points for a particular run of NSGA-II on WFG2 are respectively indicated by the circles in upper and lower panels of Figure \ref{fig:yFirstRunMag_WFG2}. Shaded bands indicate one standard deviation. All panels have the same horizontal axis, viz., the number of timesteps (from 0 to $N = 10$). The vertical axes are $[1-\Delta,1+\Delta]$, where $\Delta$ is shown below each panel. Not shown explicitly is the average weighting of non-dominated points, i.e., the red curve divided by the black one, but so long as the colored bands already shown are visibly separate, this consistently lies above the blue band. 
(Bottom) As for the top panels, but for SPEA2.}
  \label{fig:yEnsembleCharacterization}
\end{figure}


\begin{figure}[h]
  \centering
  \includegraphics[trim = 10mm 115mm 10mm 115mm, clip, width=\columnwidth,keepaspectratio]{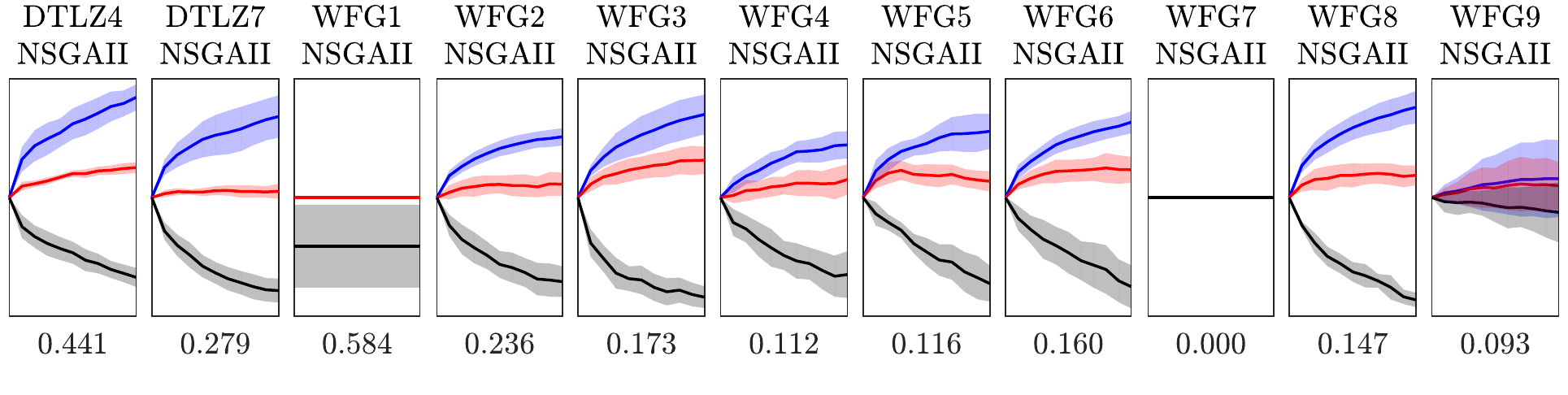} \\
  \includegraphics[trim = 10mm 120mm 10mm 120mm, clip, width=\columnwidth,keepaspectratio]{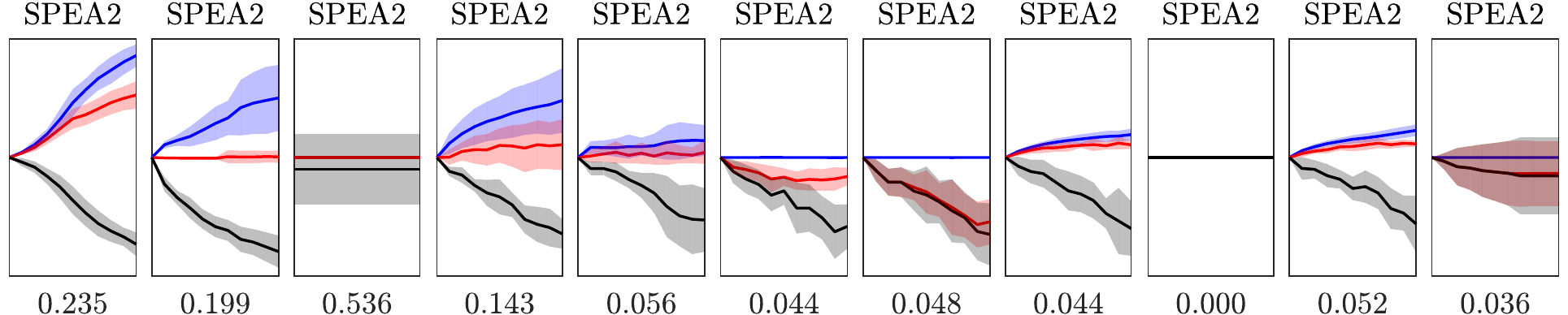} 
  \caption{As in Figure \ref{fig:yEnsembleCharacterization}, but for diversity taken as the magnitude at the scale maximizing the quotient by the initial timestep.}
  \label{fig:zEnsembleCharacterization}
\end{figure}

Although Figure \ref{fig:yEnsembleCharacterization} shows that the weighting gradient flow causes a significant proportion of points to become dominated by others, Figure \ref{fig:igdEnsembleCharacterization} uses the \emph{inverted generational distance} (IGD) relative to uniformly distributed reference points on Pareto fronts \cite{tian2018sampling} to show that this qualitative change in dominance is belied by only minor quantitative changes in the distance to Pareto fronts. 
\footnote{
Recall that the IGD of an objective point set $X$ relative to a reference point set $R$ is $\frac{1}{|R|} \sum_{r \in R} \min_{x \in X} d(x,r)$.
}
(Note that the relatively large increases in IGD for DTLZ4 and DTLZ7 are consequences of starting from a low baseline.) That is, feasible points give a better quantitative picture of diversification performance than nondominated points, especially in light of use cases in which the weighting gradient flow is not limited to postprocessing.

\begin{figure}[h]
  \centering
  \includegraphics[trim = 10mm 115mm 10mm 115mm, clip, width=\columnwidth,keepaspectratio]{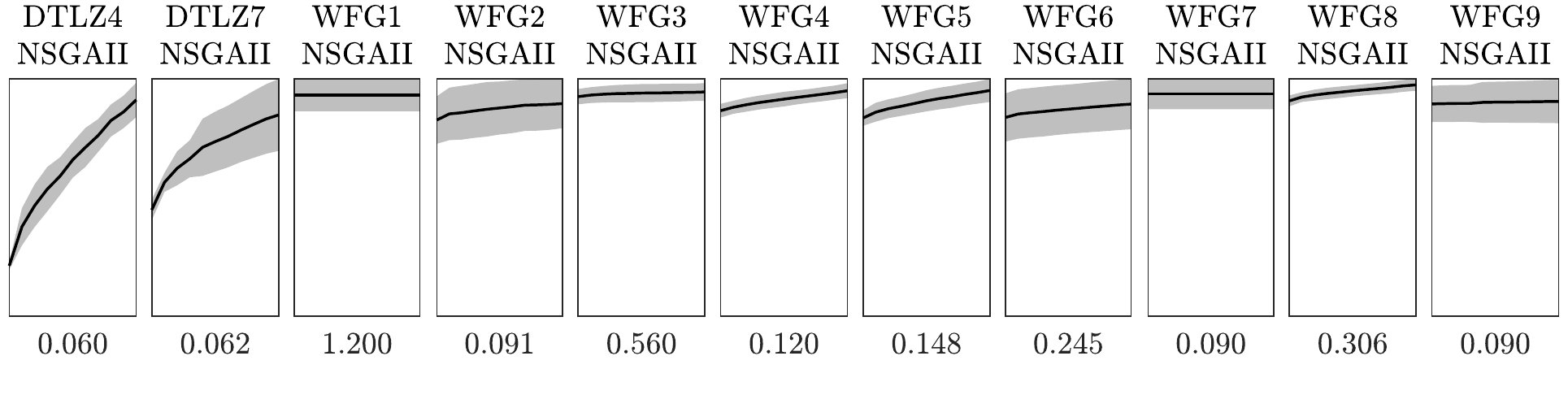} \\
  \includegraphics[trim = 10mm 120mm 10mm 120mm, clip, width=\columnwidth,keepaspectratio]{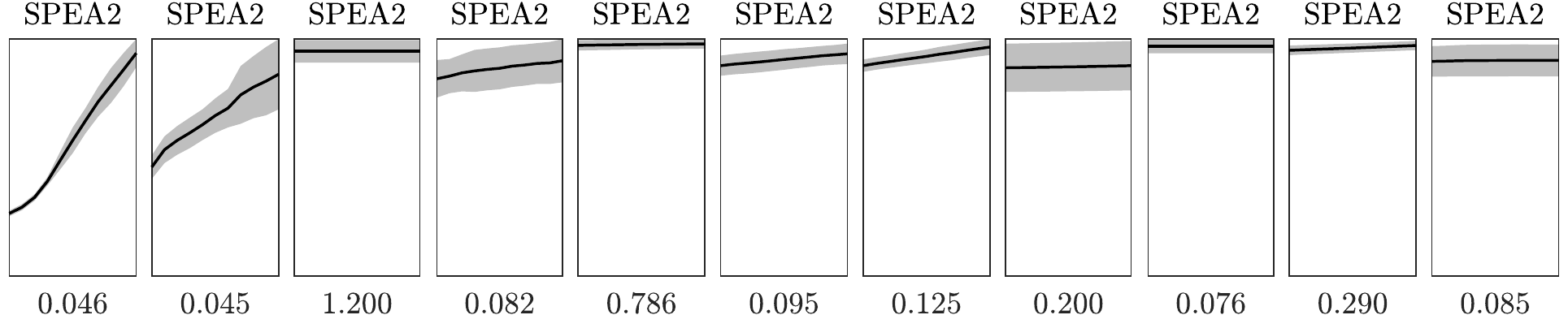} 
  \caption{The weighting gradient flow only slightly affects the quantitative dominance behavior of points, as measured by IGD. (Top) IGD under the weighting gradient flow starting from the results of NSGA-II runs, using uniformly distributed reference points on Pareto fronts. Shaded bands indicate one standard deviation. All panels have the same horizontal axis, viz., the number of timesteps (from 0 to $N = 10$). The vertical axes are $[0,y]$, where $y$ is shown below each panel. 
  (Bottom) As for the top panels, but for SPEA2.
   }
  \label{fig:igdEnsembleCharacterization}
\end{figure}

Rather than relying solely on a delicate characterization of diversity, we also visualize some of the results directly: this is the rationale for three-objective problems. Indeed, Figures \ref{fig:yFirstRunFlow_NSGAII_WFG2}-\ref{fig:yFirstRunFlow_NSGAII_WFG3} show how diversity in objective space is promoted for WFG2-3. Figures \ref{fig:yFirstRunFlow_SPEA2_WFG2}-\ref{fig:yFirstRunFlow_SPEA2_WFG3} show analogous results for SPEA2.

Careful inspection reveals that the weighting gradient flow tends to induce a gap between the boundary of the non-dominated region and its interior, which is consistent with the generally observed phenomenon that the largest weights in finite subsets of Euclidean space tend to occur on boundaries and the smallest weights immediately ``behind'' the boundary. Meanwhile, the boundary region tends to become slightly more populated. 
\footnote{
This highlights the need to distinguish between diversity and uniformity. In fact, the maximally diverse probability distribution on the interval $[0,L]$ is $\frac{1}{2+L}(\delta_0 + \lambda|_{[0,L]} + \delta_L)$, where Dirac measures and a restriction of Lebesgue measure are indicated on the right hand side \cite{leinster2019maximum}. Thus it is only in a suitable limit that boundary effects can be ignored in considerations of diversity. 
}
From the perspective of a MOEA, this is frequently a benefit, since extremal and non-extremal points on the non-dominated approximation of the Pareto front frequently carry different practical significance.
\footnote{
Using a scale $t > t_+$ for the weighting gradient flow would tend to diminish the distinction between uniformity (which is not a function of scale) and diversity (which is a function of scale). In other words, our experiments deliberately magnify this distinction to the greatest possible extent.
}

\begin{figure}[h]
  \centering
  \includegraphics[trim = 33mm 90mm 115mm 85mm, clip, width=.49\columnwidth,keepaspectratio]{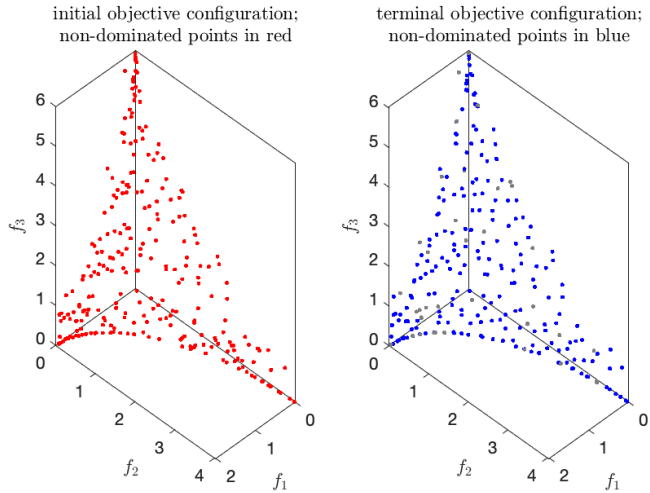}
  \includegraphics[trim = 113mm 90mm 35mm 85mm, clip, width=.49\columnwidth,keepaspectratio]{yFirstRunFlow_NSGAII_WFG220210222.pdf}
  \caption{(Left) Initial configuration in objective space for WFG2 after a NSGA-II run. (Right) Terminal configuration for WFG2 after subsequently evolving under the weighting gradient flow. Dominated points are gray.}
  \label{fig:yFirstRunFlow_NSGAII_WFG2}
\end{figure}

\begin{figure}[h]
  \centering
  \includegraphics[trim = 33mm 90mm 115mm 85mm, clip, width=.49\columnwidth,keepaspectratio]{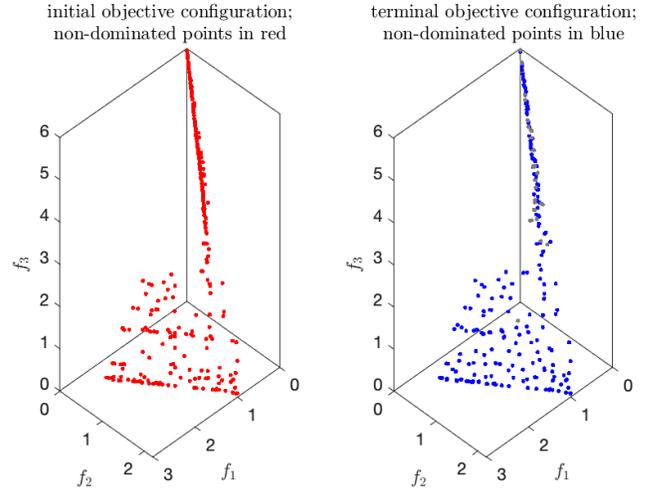}
  \includegraphics[trim = 113mm 90mm 35mm 85mm, clip, width=.49\columnwidth,keepaspectratio]{yFirstRunFlow_NSGAII_WFG320210222.pdf}
  \caption{As in Figure \ref{fig:yFirstRunFlow_NSGAII_WFG2}, but for WFG3.}
  \label{fig:yFirstRunFlow_NSGAII_WFG3}
\end{figure}

\begin{figure}[h]
  \centering
  \includegraphics[trim = 33mm 90mm 115mm 85mm, clip, width=.49\columnwidth,keepaspectratio]{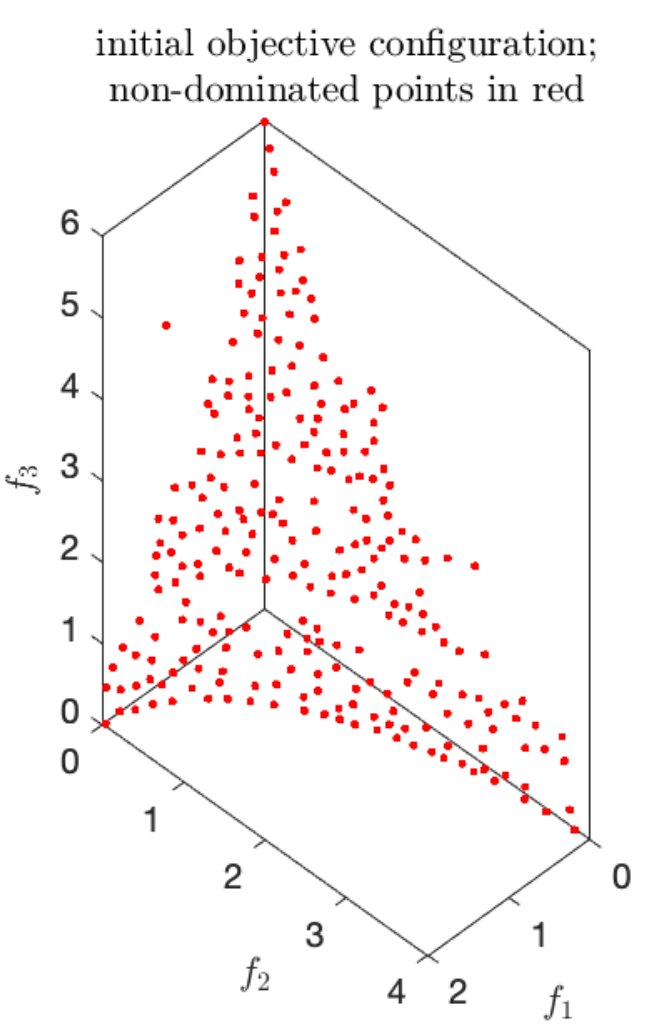}
  \includegraphics[trim = 113mm 90mm 35mm 85mm, clip, width=.49\columnwidth,keepaspectratio]{yFirstRunFlow_SPEA2_WFG220210223.pdf}
  \caption{As in Figure \ref{fig:yFirstRunFlow_NSGAII_WFG2}, but for SPEA2.}
  \label{fig:yFirstRunFlow_SPEA2_WFG2}
\end{figure}

\begin{figure}[h]
  \centering
  \includegraphics[trim = 33mm 90mm 115mm 85mm, clip, width=.49\columnwidth,keepaspectratio]{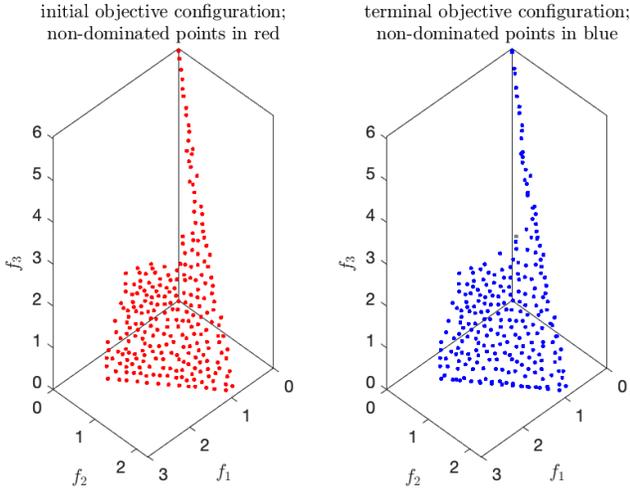}
  \includegraphics[trim = 113mm 90mm 35mm 85mm, clip, width=.49\columnwidth,keepaspectratio]{yFirstRunFlow_SPEA2_WFG320210223.pdf}
  \caption{As in Figure \ref{fig:yFirstRunFlow_NSGAII_WFG3}, but for SPEA2.}
  \label{fig:yFirstRunFlow_SPEA2_WFG3}
\end{figure}

\section{\label{sec:Extensions}Algorithmic extensions}

\subsection{\label{sec:MOWGF}Multi-objective weighting gradient flow}

We can combine the weighting gradient flow with a multi-gradient descent strategy in a way somewhat akin to \cite{desideri2012multiple}. The basic ideas that build on the weighting gradient flow are:
\begin{itemize}
	\item Introduce variable regularizing terms $\lambda_w$ and $\lambda_f$ for the weighting and function gradient flows, respectively; 
	\item Form the objective-space differentials $dy_j = ds \cdot [\lambda_w S_j (\hat \nabla w)_j + \lambda_f \sum_\ell (\hat \nabla f)_\ell ]$, where the sum is over $\ell$ such that $\langle (\hat \nabla w)_j,(\hat \nabla f_\ell)_j \rangle > 0$.
\end{itemize}

While we have tried this technique in isolation on MOEA benchmark problems, the results are poor. However, this should come as no surprise: the benchmark problems are designed to frustrate MOEAs, to say nothing of techniques involving gradients.

\subsection{\label{sec:Recycling}Recycling function evaluations}

In our experiments with post-processing the output of MOEAs, the weighting gradient flow evolution took time comparable to (and in the case of NSGA-II, slightly more than) the MOEA itself. Most of the time is spent evaluating the fitness function: apart from an initialization step, the evaluations are performed to compute Jacobians in service of pullback operations, and a lesser number are performed to compute pushforwards to maintain consistency. 

However, our motivating problems require significant time (on the order of a second) for function evaluations. This demands a more efficient pullback scheme that minimizes (or avoids altogether) function evaluations, even if the results are substantially worse. A reasonable idea is ``recycling'' in a sense similar to that employed in some modern Monte Carlo algorithms \cite{frenkel2004speed}. Specifically, rather than computing a good approximation to the Jacobian by evaluating functions afresh at very close points along coordinate axes, settle instead for an approximation of lesser quality that exploits existing function evaluations. We have implemented this approach in concert with a \emph{de novo} computation of the Jacobian in the event that this initial Jacobian estimate (IJE) does not have full rank. Our experiments suggest that this approach works reasonably well: for a typical run from \S \ref{sec:Benchmark}, the number of function evaluations is reduced from 30250 to 2750, and the actual results are broadly comparable (sometimes better, sometimes worse): see Figures \ref{fig:fastEnsembleCharacterization} and \ref{fig:fastIgdEnsembleCharacterization}.

\begin{figure}[h]
  \centering
  \includegraphics[trim = 10mm 115mm 10mm 115mm, clip, width=\columnwidth,keepaspectratio]{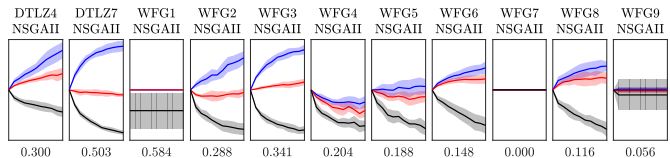} 
  \caption{As in the top panel of Figure \ref{fig:yEnsembleCharacterization}, but for a Jacobian approximation that uses existing function evaluations, increasing speed at the cost of accuracy.}
  \label{fig:fastEnsembleCharacterization}
\end{figure}

\begin{figure}[h]
  \centering
  \includegraphics[trim = 10mm 120mm 10mm 115mm, clip, width=\columnwidth,keepaspectratio]{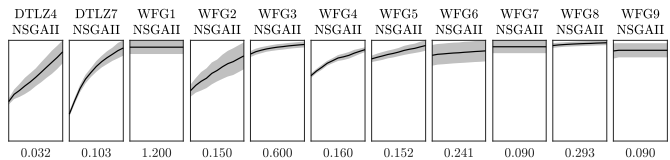} 
  \caption{As in Figure \ref{fig:igdEnsembleCharacterization}, but for a Jacobian approximation that uses existing function evaluations, increasing speed at the cost of accuracy.}
  \label{fig:fastIgdEnsembleCharacterization}
\end{figure}

This strategy is bound to work poorly if our evaluation points lie on a manifold of nonzero codimension or low curvature, because in such circumstances a matrix that transforms vectors from a base point to evaluation points into (a small multiple of) the standard basis will have large condition number. 
\footnote{
A more refined approach would be to exploit the ``good'' points, evaluating the function afresh in new directions in order to get reasonably well-behaved Jacobian estimates. A rough sketch of this approach is to compute a SVD of the IJE, then use this in turn to estimate the rank. If the rank is not full, then revise the IJE as follows. First, the pivots for the reduced row echelon form of the IJE yield a maximal linearly independent subset of the columns. To augment these columns, reuse the SVD to produce their orthocomplement (which equals the kernel of the transpose of the IJE, which itself is determined by the transpose of the rows of the first term in the SVD that correspond to zero singular values). After performing any necessary function evaluations in the orthocomplement, finally apply a linear transformation that yields a revised Jacobian estimate in the standard basis. This more refined approach is sufficiently intricate that we have not implemented it. 
}
However, these situations are relatively unlikely to present major problems in practice, and the recycling approach is likely to be useful in most if not all situations where function evaluations are expensive.

\subsection{\label{sec:Spread}Using spread instead of magnitude}

A more principled approximation of the magnitude than \cite{ulrich2010defining,ulrich2011maximizing} use is the \emph{spread} \cite{willerton2015spread}. In some respects, the spread is even better behaved than the magnitude and requires only $O(n^2)$ operations with a naive algorithm. The idea is to take the \emph{Leinster-Cobbold diversity} \cite{leinster2012measuring} of order zero for the uniform distribution, i.e.
\begin{equation}
\label{eq:spread}
E_0(d) := \sum_j \frac{1}{\sum_k \exp(-d_{jk})}.
\end{equation}
It turns out that $E_0(d)$ is bounded above by the maximum diversity of $d$ (which itself is bounded above by magnitude if $d$ is positive definite) and also by $\exp \max_{jk} d_{jk}$. Furthermore, $E_0(td)$ increases from $1$ at $t = -\infty$ to $n$ at $t = \infty$. Finally, $E_0(d)$ itself admits a good approximation by truncating the sums involved, and with some care this can reduce the computational effort compared to the naive algorithm.

If we define $a_j := 1/\sum_k \exp(-d_{jk})$, then $E_0(d) = \sum_j a_j$, and we can consider the gradient flow obtained by using the ``spread vector'' $a$ in place of a weighting. However, this does not yield useful results, as shown in Figure \ref{fig:spreadEnsembleCharacterization}.

\begin{figure}[h]
  \centering
  \includegraphics[trim = 10mm 115mm 10mm 115mm, clip, width=\columnwidth,keepaspectratio]{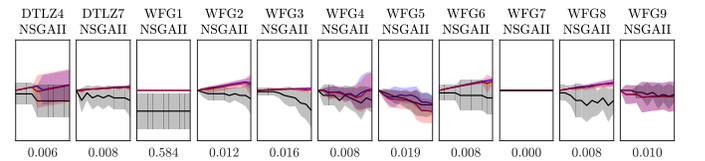} 
  \caption{As in the top panel of Figure \ref{fig:yEnsembleCharacterization}, but for a gradient flow using the ``spread vector'' $a$ in lieu of a weighting at scale $t_+$. Not shown: dilating $d$ by $t_+$ does not have any substantial effect.}
  \label{fig:spreadEnsembleCharacterization}
\end{figure}

\subsection{\label{sec:OtherConsiderations}Other considerations}

At present, we have not explored stochastic gradient descent and/or line search techniques to improve performance: however, it seems likely that these would be at best delicate to implement successfully. While we could restrict the summations involved in the weighting gradient flow, these have a comparatively small impact on runtime. Other avenues for acceleration appear to be restricted to computing or approximating weightings more efficiently. There is much that can be done in this vein, but we omit discussion for economy.

It seems fruitful to consider non-differentiable variants of the weighting gradient flow (for this, see \S \ref{sec:DiscreteObjective}). With this in mind, we recall an abstract perspective on genetic algorithms provided by \cite{mumford2010pattern}, which points out in turn that a sequential Monte Carlo  (SMC; e.g., a particle filter) algorithm is essentially a Markov chain Monte Carlo (MCMC) algorithm operating over a population versus an individual, and that a genetic algorithm is essentially a SMC algorithm that makes use of genetic-type operations.
\footnote{
Nominally, these are usually thought of as, e.g. mutation, crossover, and selection. But at a more abstract level, a genetic algorithm sampling from a Markov random field has an objective in which all of the terms involve local interactions between variables. A reasonable genetic-type operation must respect the topology of these interactions, as encoded in a factor graph (or equivalently, the abstract simplicial complex encoding the relation between factors and variables). In other words, the operation must involve a reasonable mechanism for segmentation, i.e., partitioning variables and manipulating them in a way that is compatible with the partition.
}
Now a MOEA is essentially a genetic algorithm that addresses multiple objectives. 

From this perspective, and abstracting some of the techniques we have developed, we can envision the broad outlines of a MOEA that leverages weightings in a principled and coherent set of internal mechanisms with the aim of outperforming existing algorithms:
\begin{itemize}
	\item Population: use an existing SMC algorithm or extend a fast multiple-proposal MCMC technique such as \cite{huntsman2020fast} to accept multiple proposals and maintain a population;
	\item Genetic-type operations: intersperse general-purpose (and possibly also domain-specific) genetic-type operations with a mechanism for favoring increases to individual weighting components (e.g., eliminate points with small absolute values of weighting components and replenish the population with ``spores'' distributed around points with large weighting components);
	\item Multiple objectives: compare state-space moves that increase weighting components with those that favor objectives, and select moves that accomplish both, with regularization terms to make an explore (weighting)-exploit (objective) tradeoff.
\end{itemize}

\section{\label{sec:DiscreteObjective}Discrete objective space}

Suppose that $f : X \rightarrow Y \subset \mathbb{R}^m$ for $Y$ countable and that $\mathbb{R}^m$ is endowed with a suitable distance $d$ which need not be a metric, let alone Euclidean. Now the weighting gradient \eqref{eq:gradient} still makes sense on $Y$, though the weighting gradient flow \eqref{eq:weightingGradientFlow} does not make sense any more. However, for $x \in X$ and $\delta x$ a suitable random perturbation, comparing $f(x+\delta x)$ and $f(x) + \delta s \cdot \hat \nabla w|_{f(x)}$ over multiple realizations of $\delta x$ 
yields a stochastic analogue of (the pullback of) \eqref{eq:weightingGradientFlow}:
\begin{equation}
\label{eq:stochasticWeightingGradientFlow}
\hat \delta x|_x := \arg \min_{\delta x} d \left ( f(x+\delta x) , f(x)+\delta s \cdot \hat \nabla w|_{f(x)} \right ).
\end{equation}

A sensible heuristic for $\delta s$ in \eqref{eq:stochasticWeightingGradientFlow} is to enforce something like $\langle |\delta s \cdot \hat \nabla w | \rangle = \delta d/2$ initially, where $\delta d$ is the average Euclidean distance between pairs of initial points in $Y$ with minimal nonzero $d$. That is, take $\delta s = \delta d / 2 \langle | \hat \nabla w | \rangle$. This would seem to encourage 
the movement of significant numbers of points while a significant fraction of points would also remain stationary for the first timestep. Meanwhile, a reasonable heuristic for $\delta x$ in the case $X \subset \mathbb{R}^M$ is to take $\delta x \sim \mathcal{N}(\mu,\Sigma)$, where $\mu = D_x f \backslash (\hat \nabla w|_{f(x)})$, $\Sigma = \Delta ( | \mu |,\delta_j,\dots,\delta_j )$ in an(y) orthonormal basis with first vector $\mu$, 
\footnote{
To produce a matrix $B$ whose columns form an orthonormal basis with first vector $a$ in MATLAB, use \texttt{ahat = a(:).'/norm(a); B = [ahat;null(ahat).']';}.
}
and where $\delta_j := \frac{1}{2}\min_{k \ne j} |x_j-x_k|$.

Note that while \eqref{eq:stochasticWeightingGradientFlow} can be approximated using a metaheuristic, the spirit of applications is to take a relatively small number of candidate perturbations $\delta x$ versus finding the optimal perturbation \emph{per se}. With this in mind, define the \emph{effort} $E(x)$ to be the number of candidate perturbations for $x$. If we can neglect any effects of perturbing the same points over successive timesteps, a sensible heuristic for effort is
\begin{equation}
\label{eq:stochasticWeightingGradientFlowEffort}
E(x) = \left \lceil C \frac{ \left | \hat \nabla w|_{f(x)} \right |}{\sum_{x'} \left | \hat \nabla w|_{f(x')} \right |} \right \rceil
\end{equation}
for a constant $C$ that approximately determines the overall effort. This heuristic simply distributes effort to each point according to its ``need'' to be perturbed as measured by its weighting gradient.

If points occupy the same positions for a long time (say, because the geometry of $f$ inhibits their movement), another sensible heuristic attuned to this behavior is an effort defined in terms of 
the effective energies of points \emph{\`a la} \cite{huntsman2021sampling}, using the cumulative occupation times of positions to define probabilities and some notion of (e.g.) convergence rate (which in general will vary over time) to define the overall timescale required for the framework. This heuristic will focus effort on positions that have been occupied the least amount of time while encouraging all positions to be uniformly occupied. 

Shifting gears, consider the situation where we have no objective $f$ and $X$ is completely generic, so that even the weighting gradient is not necessarily defined. Suppose we have $\{x_j\}_{j = 1}^n \subset X$ and distance matrix $d_{jk} = d(x_j,x_k)$. Let $w$ be the weighting at $t_+$ and select an index $j_*$ with small weighting component $w_{j_*}$. Now sample $x'$ from a ``proposal'' distribution $\mathbb{P}(x'|x_{j_*})$ and compare the magnitudes $\mu := \text{Mag}(t_+; \{x_j\}_{j = 1}^n)$ and $\mu' := \text{Mag}(t_+; \{x_j\}_{j \ne j_*} \cup \{x'\})$, where we use the same (original) value of $t_+$ in both instances. If $\mu' \ge \mu$ then replace $x_{j_*}$ with $x'$; also, if $\mu' < \mu$ and $x' \not \in \{x_j\}_{j \ne j_*}$, do the same thing with probability $\exp(-\beta [\mu-\mu'])$ for some suitable $\beta < 0$. That is, perform a Metropolis-Hastings step according to the proposal.
\footnote{
More generally, we can make multiple proposals for $x'$ and use one of the acceptance mechanisms described in \cite{huntsman2020fast}.
}
This sort of technique works in great generality, and its efficiency is determined largely by that the the proposal.

\begin{figure}[h]
  \centering
  \includegraphics[trim = 35mm 100mm 35mm 100mm, clip, width=\columnwidth,keepaspectratio]{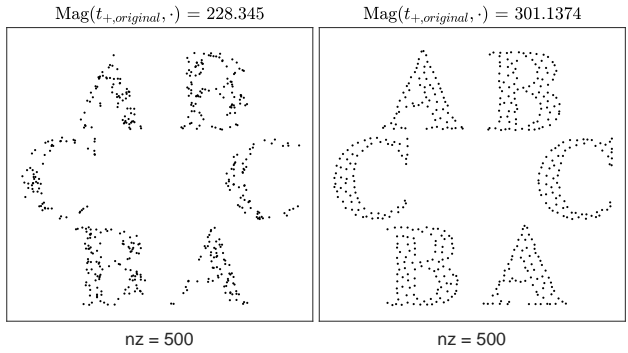}
  \caption{(Left) 500 points sampled as in Figure \ref{fig:boundaryABC20210402}. (Right) The result of the following process: for each of 1000 timesteps, we select the 10 points with least weighting component (at the original value of $t_+$) and sample 10 corresponding candidate points outside of the current locations and without replacement; then (and in ascending order of weighting component) test to see if the overall magnitude increases upon substituting the candidate point (i.e., perform a Metropolis-Hastings step for $\beta = \infty$). The magnitude increases by almost 32 percent.}
  \label{fig:MaxDiversityABC20210405}
\end{figure}

\section{\label{sec:Remarks}Remarks}

Although our experiments have focused on the results of applying the weighting gradient flow and related constructions after a MOEA has been applied, the more natural application is in the course of a MOEA. As mentioned in \S \ref{sec:Extensions} and \ref{sec:DiscreteObjective}, there is ample scope to refine and build on ideas for increasing weighting components in specific  contexts. It is nevertheless clear that the theory of magnitude informs principled and practical diversity-promoting mechanisms that can already be usefully applied to benchmark multi-objective problems.


%

\appendices
\section{\label{sec:erosion}Erosion}

In the setting of Euclidean space we can iteratively downsample a finite set $B$ to obtain $B' \subseteq B$ with positive weights. The idea is as follows. Let $Z_{jk} := \exp(-d(b_j,b_k))$ and let $w$ be the unique weighting for $Z$. 
\footnote{
In an arbitrary metric space, we may not have a weighting at all, let alone a unique one. However, the absence or nonuniqueness of any weighting is a degenerate pathology that can be avoided by perturbations. 
}
Now set $B' = B$, $w' = w$, and repeatedly reassign $B' \leftarrow \{b_j \in B': w'_j > 0\}$ and recompute $Z'$ and $w'$ until the corresponding weighting is positive. This iteration terminates in a unique $B' \ne \varnothing$ admitting a positive weighting. When normalized, this weighting maximizes diversity on $B'$, as depicted in Figure \ref{fig:erosionABC20210402}. 

\begin{figure}[h]
  \centering
  \includegraphics[trim = 40mm 110mm 40mm 110mm, clip, width=\columnwidth,keepaspectratio]{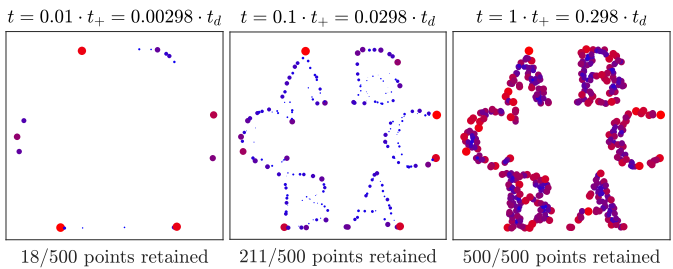}
  \caption{From left to right: diversity-saturating weightings on erosions of the set from Figure \ref{fig:boundaryABC20210402} for varying scale factors $t$. The number of points retained in each erosion is shown. Both the color and size of a point is a function of the weighting component.}
  \label{fig:erosionABC20210402}
\end{figure}

In short, we can enforce the existence of a positive weighting by scaling or eroding, depending on whether the scale or elements of the data should be prioritized. Provided we ignore (points with) lesser weights, the two approaches appear to yield qualitatively similar results.


\section*{Acknowledgment}

Thanks to Andy Copeland, Megan Fuller, Zac Hoffman, Rachelle Horwitz-Martin, and Daryl St. Laurent for many patient questions and observations that clarified and simplified the ideas herein. This research was developed with funding from the Defense Advanced Research Projects Agency (DARPA). The views, opinions and/or findings expressed are those of the author and should not be interpreted as representing the official views or policies of the Department of Defense or the U.S. Government. DISTRIBUTION STATEMENT A. Approved for public release; distribution is unlimited.

\ifCLASSOPTIONcaptionsoff
  \newpage
\fi



\bibliographystyle{IEEEtran}

%
%
%

%


\begin{IEEEbiographynophoto}{Steve Huntsman} 
is a mathematician in the cyber division at STR. 
\end{IEEEbiographynophoto}






\end{document}